\documentclass[a4paper,11pt]{article}
\usepackage[margin=1in]{geometry}

\usepackage[colorlinks,linkcolor=blue,citecolor=blue]{hyperref}
\usepackage{amsmath,amssymb,amsthm}
\usepackage{xcolor}
\usepackage[round]{natbib}

\title{Online Learning with Low Rank Experts%
\footnote{Accepted for presentation at Conference on Learning Theory (COLT) 2016.}}

\author{%
Elad Hazan%
\footnote{Princeton University; email: \texttt{ehazan@cs.princeton.edu}. 
Parts of this work were done while at Microsoft Research, Herzliya.}
\and
Tomer Koren%\footnotemark[2]
\footnote{Technion---Israel Institute of Technology and Microsoft Research, Herzliya; email: \texttt{tomerk@technion.ac.il}.}
\and
Roi Livni%
\footnote{The Hebrew University of Jerusalem and Microsoft Research, Herzliya; email: \texttt{roi.livni@mail.huji.ac.il}.}
\and 
Yishay Mansour%\footnotemark[2]
\footnote{Microsoft Research, Herzliya and Tel Aviv University; email: \texttt{mansour.yishay@gmail.com}.}
}

\usepackage{mathabx}
\usepackage{algorithm,algorithmic}

\theoremstyle{plain}
\newtheorem{theorem}{Theorem}
\newtheorem{lemma}[theorem]{Lemma}

\newtheorem*{theorem*}{Theorem}
\newtheorem*{lemma*}{Lemma}
\newtheorem*{corollary*}{Corollary}
\newtheorem*{proposition*}{Proposition}
\newtheorem*{claim*}{Claim}
\newtheorem*{fact*}{Fact}
\newtheorem*{observation*}{Observation}

\theoremstyle{definition}

\newtheorem*{definition*}{Definition}
\newtheorem*{remark*}{Remark}
\newtheorem*{example*}{Example}

\newtheorem{question}{Open Problem}

\newcommand{\ignore}[1]{}

\newcommand{\reals}{\mathbb{R}}
\newcommand{\regret}{\mathrm{Regret}}

\renewcommand{\eqref}[1]{Eq.~(\ref{#1})}
\newcommand{\algref}[1]{Algorithm~\ref{#1}}
\newcommand{\lemref}[1]{Lemma~\ref{#1}}
\newcommand{\thmref}[1]{Theorem~\ref{#1}}
\newcommand{\secref}[1]{Section~\ref{#1}}
\newcommand{\apref}[1]{Appendix~\ref{#1}}
\newcommand{\F}{\mathcal{F}}
\newcommand{\E}{\mathbb{E}}
\newcommand{\symplex}{\Delta}
\newcommand{\bm}[1]{\mathbf{#1}}
\newcommand{\set}[1]{\{#1\}}
\newcommand{\mvee}[1]{\mathrm{MVEE}(#1)}

\newcommand{\cE}{\mathcal{E}}
\newcommand{\pinv}{^\dagger}
\newcommand{\tr}{^\top}

\newcommand{\epsrank}{\mathop{\mathrm{rank}_\epsilon}}
\newcommand{\Del}{\symplex}
\newcommand{\norm}[1]{\|#1\|}
\newcommand{\U}{U}
\renewcommand{\L}{L}
\newcommand{\ee}{\mathbf{e}}
\newcommand{\ldim}{\mathrm{Ldim}}
\renewcommand{\top}{{\mkern-1.5mu\mathsf{T}}}

\newcommand{\ym}[1]{{\textcolor{magenta}{}}}
\newcommand{\rl}[1]{{\textcolor{purple}{}}}
\newcommand{\tk}[1]{{\textcolor{cyan}{}}}
\newcommand{\eh}[1]{{\textcolor{orange}{}}}

\begin{document}

\maketitle

\begin{abstract}
We consider the problem of prediction with expert advice when the losses of the experts have low-dimensional structure: they are restricted to an unknown $d$-dimensional subspace.
We devise algorithms with regret bounds that are independent of the number of experts and depend only on the rank $d$.
For the stochastic model we show a tight bound of $\Theta(\sqrt{dT})$, and extend it to a setting of an approximate $d$ subspace.
For the adversarial model we show an upper bound of $O(d\sqrt{T})$ and a lower bound of $\Omega(\sqrt{dT})$.
\end{abstract}

%\begin{keywords}
%List of keywords
%\end{keywords}

\section{Introduction}
Arguably the most well known problem in online learning theory is the so called \emph{prediction with experts advice} problem. In its simplest form, a learner wishes to make an educated decision and at each round chooses to take the advice of one of $N$ experts. The learner then suffers a loss between $0$ and $1$. \tk{goal is to minimize her regret, which is ...}

It is a standard result in online learning that, without further assumptions, the best strategy for the learner will incur $\Theta(\sqrt{T \log N})$ regret~\citep{CesaBianchiLugosi}. However, it is natural to assume that while experts are abundant, their decisions are based on common paradigms and that their decision making is based on few degrees of freedom -- for example, if experts are indeed experts, their political bias, social background or school of thought largely dominates their decision making. Experts can also be assets on which the learner wishes to distribute her wealth. In this setting, weather, market condition and interests are dominant factors.

It is also sensible to assume that one can exploit this structure to achieve better regret bounds, potentially independent of the actual number of experts while still maintaining a strategy of picking an expert's advice at each round. Our main result is of this flavor and we show how a learner can exploit hidden structure in the problem in an online setting.

We model the problem as follows: We assume that each expert corresponds to a vector $u_i$ in $\reals^d$ space where $d$ is potentially small. Then at each round the experts loss corresponds to a scalar product with a vector $v_t$ chosen arbitrarily, and possibly in an adversarial manner. The learner does not observe the chosen embedding of the experts in Euclidean space nor the vectors $v_t$, and can only observe the loss of each expert.

To further motivate our setting, let us consider the low rank expert model in the stochastic case. It is well known that for linear predictors in $d$-dimensional space the regret will be $O(\sqrt{dT})$, independent of the number of experts. Indeed, we show that a simple follow the leader algorithm will achieve this regret bound. In fact, one novelty of this paper is a regret bound that depends on an approximate rank -- formally we show that one can improve on the $O(\sqrt{T\log N})$ regret bound and derive bounds that depend on the \emph{approximate rank} rather than the number of experts.

The non-stochastic setting is more challenging. It is true that for linear predictors in $d$-dimension one can achieve $O(\sqrt{dT})$ regret bound even in the non-stochastic case. But the result assumes that learner has access to the geometric structure of the problem, namely, the embedding of the experts in the Euclidean space. Given the embedding one can apply a \emph{Follow the Regularized Leader} approach with proper regularization to derive the desired regret bound.

Our main result is a regret minimization algorithm that achieves an $O(d\sqrt{T})$ regret in this low $d$-rank setting, when the learner
{\em does not} have access to the experts' embedding in Euclidean space.
Our algorithm does not need to know the value of the rank $d$, and adaptively adapts to it.
Thus we demonstrate a regret bound that is independent of number of experts.
We accompany this upper bound with an $\Omega(\sqrt{dT})$ lower bound.

Our results are part of a larger agenda in online learning.
A working premise in Online Learning is that in most cases the stochastic case is the hardest case.
Indeed, the literature is filled with generalization bounds and their analogue regret bounds.
However, a striking difference is that the statistical bounds are often achieved using simple ERM algorithms,
that are oblivious to any structure in the problem, even if the structure is required for the generalization bounds to be valid.
In contrast, to achieve the analogue regret bound, one has to work harder.
For finite hypothesis class the $\log N$ factor is achieved by a sophisticated  algorithm,
and for more general convex problems in Euclidean space a problem-specific regularization needs to be invoked in order to achieve optimality.
Thus, a key difference is that online algorithms need to be tailored to the structure of the problem. This leads to the disappointing fact that to achieve optimal regret bounds, it is not enough for the problem to be structured but the learner needs to actively understand the structure.

Our current research is an attempt to better understand this key difference:
we wish to understand whether an online linear predictor can somehow exploit the geometry of the problem in an implicit manner, similarly to batch ERM algorithms, and how.
For this, we invoke a setting where the learner must choose its predictor without the a-priori ability to devise a regularizer.
Our findings so far indeed demonstrate that even without access to the structure the learner can indeed overcome her dependence on the irrelevant parameter $N$.

Technically, one should compare our regret bound of $O(d\sqrt{T})$ to the standard regret bound of $O(\sqrt{T \log N})$.
For our bound to be superior one needs that $d =o (\sqrt{\log N})$;
while this can indeed be the case in various settings, our result can be better seen as a first step in a more general research direction.
We aim to understand how online algorithms can take advantage of structural assumptions in the losses, without being given any explicit information about it.

\subsection{Related Work}

Low rank assumptions are ubiquitous in the Machine Learning literature. They have been successfully applied to various problems, most notably to matrix completion (\citealp{candes2009exact, foygel2011concentration, srebro2004maximum}) but also in the context of classification with missing data (\citealp{goldberg2010transduction,hazan2015missing}) and large scale optimization \citep{shalev2011large}.

A similar problem that was studied in the literature is the \emph{Branching Experts Problem}~\citep{gofer2013regret}. In the branching expert problem $N$ potential experts are effectively only $k$ distinct experts, but the clustering of the experts to the $k$ clusters is unknown a-priori. This case can be considered as a special instance of our setting as indeed we can embed each expert as a $k$-dimensional vector. \cite{gofer2013regret} proved a sharp $\Theta(\sqrt{kT})$ regret (the bound is tight only when $k < c\log{N}$ for some constant $c>0$). It is perhaps worth noting that when effectively only $k$ experts appear, the stochastic bound is $O(\sqrt{T \log k})$, thus showing that in this similar problem, it is not true that the stochastic case is the hardest case.

\paragraph{Complexity measures for online learning.}

We are not the first to try and understand what is the proper analogue for ERM in the online setting. Notions like the VC-dimension and Rademacher complexities have been extended to notions of Littlestone-Dimension (\citealp{littlestone1988learning,shalev2011online}), and Sequential Rademacher Complexity \citep{rakhlin2010online} respectively.

The SOA algorithm suggested by \cite{ben2009agnostic} is a general framework for regret minimization that depends solely on the Littlestone dimension. However, the SOA algorithm is conceptually distinct from an ERM algorithm within our framework:
%It is an algorithm where having small Littlestone dimension is not enough:
to implement the SOA algorithm, one has to have access to the structure of the class (specifically, one needs to compute the Littlestone dimension of subclasses within the algorithm).

Sequential Rademacher complexity seems like a powerful tool for improving our bounds and answering some of our open problems. There are also advances in constructing effective algorithms within this framework \citep{rakhlin2012relax}. However, as the branching expert example shows, there is no general argument that show that structure in the problem leads to stochastic--analogue bounds on the complexity.

\paragraph{Learning from easy data.}

In another line of research, which is similar in spirit to ours, several authors attempt to go beyond worst-case analysis in online learning, and provide algorithms and bounds that can exploit deficiencies in the data.
Work in this direction includes the study of 
%While most of the work in this direction has been focused on 
worst-case robust online algorithms that can also adapt to stochastic i.i.d.~data (e.g., \citealp{NIPS2009_3795,rakhlin2013localization, de2014follow, sani2014exploiting}), 
%several papers have explored 
as well as the exploration of
various structural assumptions that can be leveraged for obtaining improved regret guarantees (e.g., \citealp{cesa2007improved,hazan2010extracting,hazan2011better,chiang2012online,rakhlin2013online}).
However, to the best of our knowledge, low rank assumptions in online learning have not been explored in this context.

\paragraph{Adaptive online algorithms.}

Online adaptive learning methods have recently been the topic of extensive study and are effective for large scale stochastic optimization in practice.  One of the earliest and most widely used  methods in this family is the AdaGrad algorithm \citep{duchi2011adaptive}, a subgradient descent method that dynamically incorporate knowledge of the geometry of the data from earlier iterations. Our problem can be cast into an online linear optimization problem and subgradient descent methods are indeed applicable. It might seem at first sight that adapting the regularization via AdaGrad can lead to desired results. However, the analysis of the AdaGrad algorithm can only yield an $O(\sqrt{dNT})$ bound on the regret in our low-rank setting. In fact, a closer inspection reveals that the $\sqrt{N}$ factor in the latter bound is unavoidable for AdaGrad: as we show in \apref{sec:adagrad}, in our setting the regret of AdaGrad is lower bounded by $\Omega(\min\{\sqrt{N},T\})$.

\section{Problem Setup and Main Results}

We recall the standard adversarial online experts model for $T$ rounds with $N$ experts.
At each round $t=1,\ldots,T$, the learner chooses a probability vector $x_t \in \symplex_N $, where $\symplex_N$ denotes the $N$-simplex, namely the set of all possible distributions over $N$ experts,
\[\symplex_N=\left\{ x\in \reals^N \;:\; \forall i, ~ x(i) \ge 0 ~~\mbox{and}~~ \sum\nolimits_{i=1}^N x(i)=1 \right\}.\]
An adversary replies by choosing a loss vector $\ell_t \in [-1,1]^N$,%
\footnote{As will become apparent later, in our setup it is more natural to consider symmetric $[-1,1]$ loss values rather than the typical $[0,1]$ losses. The two variants of the problem are equivalent up to a simple shift and scaling of the losses---a transformation that preserves the rank of the loss matrix.}
and the learner suffers a loss
$x_t(\ell_t) = x_t\cdot \ell_t.$
The objective of the learner is to minimize her regret, which is defined as follows,
\[\regret_T ~=~ \sum_{t=1}^T x_t\cdot \ell_t - \min_{i \in [N]} \sum_{t=1}^T \ell_t(i).\]
%Where expectation is taken over the adversary's strategy.

In the stochastic online experts model, the adversary selects a distribution $\mathcal{D}$ over the loss vectors in $ [-1,1]^N$,
and at time $t$ a random $\ell_t \in [-1,1]^N$ is selected from $\mathcal{D}$. The regret is.
\[\regret_T ~=~ \sum_{t=1}^T x_t\cdot \E[\ell_t] - \min_{i \in [N]}\sum_{t=1}^T \E[\ell_t(i)]~,\]
where the expectations are taken over the random loss vectors selected from $\mathcal{D}$.

In our setting, we wish to assume that there is a structure over the experts which implies that the loss vectors are structured,
and are derived from a low rank subspace.
Therefore we will add the following constraint over the adversary:
let $L\in \reals^{N\times T}$ be the loss matrix obtained in hindsight (i.e., the $t$'th column of $L$ is $\ell_t$).
We restrict the feasible strategies for the adversary to only such that satisfy:
\[\mathrm{rank}(L) =d ~.\]
An equivalent formulation of our model is as follows: An adversary chooses at the beginning of the game a matrix $\U \in \reals^{N\times d}$,
where each row corresponds to an expert.
At round $t$ the adversary chooses a vector $v_t$, and the learner gets to observe  $\ell_t$ where $\ell_t = \U v_t$.
%\ym{Let $V$ be a  $d\times T$ matrix whose columns are $v_t$}\rl{Why?}
The objective of the learner remains the same: to choose at each round a probability distribution $x_t$ that minimizes the regret.
We stress that the learner observes only the loss vectors $\ell_t$, and does not have access to either $\U$ or the vectors $v_t$.

\subsection{Main Results}

We next state the main results of this paper:
\begin{theorem}
\label{thm:main}
The $T$-round regret of \algref{alg:main} (described in \secref{sec:upper} below) is at most $O(d\sqrt{T})$, where $d=\mathrm{rank}(L)$.
\end{theorem}
We remark that a regret upper bound of $O(\sqrt{T} \min\set{d,\sqrt{\log{N}}})$ is attainable by combining the standard multiplicative-updates algorithm  with our algorithm.%
\footnote{A standard way to accomplish this is by running the two online algorithms in parallel, and choosing between their predictions by treating them as two meta-experts in another multiplicative-weights algorithm.}
Our upper bound is accompanied by the following lower bound.
\begin{theorem}
\label{thm:lowbound}
For any online learning algorithm, $T$ and $d \le \log_2{N}$, there exists a sequence of loss vectors $\ell_1,\ldots \ell_T \in [-1,1]^N$ such that
\[ \regret_T ~=~
\sum_{t=1}^T x_t \cdot \ell_t - \min_{i \in [N]}\sum_{t=1}^T \ell_t(i)
~\ge~ \sqrt{\frac{d T}{8}}~,
\]
and $\mathrm{rank}(L)=d$.
\end{theorem}

\section{Preliminaries}

\subsection{Notation}
\label{sec:notation}

Let $I_n$ be the $n\times n$ identity matrix. Let $\bm{1}_n$ be a vector of length $n$ with all $1$ entries.
%Throughout we fix $N$ and $T$ and consider a setting with $N$ experts and $T$ rounds.
%Matrices are denoted by upper cases $U,V,L,M,\ldots$  and vectors are denoted by lower cases $u,v,\ell,m,x$.
The columns of a matrix $U$ are denoted by $u_1,u_2,\ldots$.
The $i$'th coordinate of a vector $x$ is denoted by $x(i)$.
%
%The matrix $L\in \reals^{N\times T}$ is reserved for the loss matrix.
%We will generally assume that it has a factorization of the form $L=UV$ where $U\in \reals^{N\times d}$ and $V\in \reals^{d\times T}$.
%
For a matrix $M$, we denote by $M^\dagger$ the Moore-Penrose pseudo-inverse of $M$.
For a positive definite matrix $H \succ 0$ we will denote its corresponding norm
$\|x\|_H = \sqrt{ x^\top H x},$ and its dual norm
$\|x\|^*_H =\sqrt{ x^\top H^{-1} x}.$
Given a positive semi-definite matrix $M \succeq 0$ its corresponding Ellipsoid is defined as:
\[
\cE(M) = \{x ~:~ x\tr M^\dagger x \le 1 \}~.
\]

\subsection{Ellipsoidal Approximation of Convex Bodies}
\label{sec:John}

A main tool in our algorithm is an Ellipsoid approximation of convex bodies. Recall John's theorem for symmetric zero-centered convex bodies.
% (see, for example, \citealp{ball1997elementary}):
%
\begin{theorem}[John's Theorem; e.g, \citealp{ball1997elementary}]
\label{thm:john}
Let $K$ be a convex body in $\reals^d$ that is symmetric around zero (i.e., $K=-K$).
Let $\mathcal{E}$ be an ellipsoid with minimum volume enclosing $K$.
Then:
\[ \frac{1}{\sqrt{d}}\mathcal{E} \subseteq K \subseteq  \mathcal{E}.\]
\end{theorem}
While computing the minimum volume enclosed ellipsoid is computationally hard, for symmetric convex bodies it can be approximated to within $1+\epsilon$ factor in polynomial time.
%In this work, we need only a constant approximation; specifically, a factor $\sqrt{2}$ would be satisfactory.
Specifically, given as input a matrix $A \in \reals^{N\times d}$, consider the polytope $P_A= \set{x \,:\, \| Ax \|_\infty \leq 1}$. We have the following.
%The procedure $\mvee{A}$ returns a matrix $M$.
%
% shows that $\mvee{A}$ can be computed in poly-time and has the following property.
%
\begin{theorem}[\citealp{grotschel2012geometric}, Theorem~4.6.5]
There exists a poly-time procedure $\mvee{A}$ that receives as input a matrix $A\in \reals^{N\times d}$ and returns a matrix $M$ such that %\eqref{eq:john} holds.
\begin{equation*}%\label{eq:john}
\frac{1}{\sqrt{2d}} \cE(M) \subseteq P_A \subseteq \cE(M).
\end{equation*}
\end{theorem}

%
%Our algorithm relies on a procedure $\mvee{A}$, that computes a Positive Semi-Definite (PSD) matrix $M \succ 0$ such that the ellipsoid $\cE(M)$ is a $2$-approximation to the Minimum Volume Enclosed Ellipsoid (MVEE) of the symmetric polytope $P_A= \set{x \,:\, \| Ax \|_\infty \leq 1}$ . In particular let us denote for each p.s.d matrix its corresponding Ellipsoid:
%\[\cE(M) = \{x: x M^\dagger x \le 1\}\]
%\ym{why $M^\dagger$ and not $M$? also $x^\top M x$}
%
%If $M=\mvee{A}$ then $M$ has the property that:
%\begin{equation}\label{eq:john}
%\cE(\frac{1}{2d} M) \subseteq P_A \subseteq \cE(M).\end{equation}
%
%Theorem~4.6.5 in \cite{grotschel2012geometric} shows that $\mvee{A}$ can be computed in poly-time. We summarize this as follows:
%
%\begin{theorem}
%There exists a poly-time procedure that receives as input a matrix $A\in \reals^{N\times d}$ and returns a matrix $M=\mvee{A}$ such that \eqref{eq:john} holds.
%\end{theorem}

\subsection{Online Mirror Descent}
\label{sec:omd}

Another main tool in our analysis is the well-known \emph{Online Mirror Descent} algorithm for online convex optimization. The Online mirror descent is a subgradient descent method for optimization over a convex set in $\reals^d$ that implies a regularization factor, chosen a-priori.  In \algref{alg:omd} we describe the algorithm for the special case where the convex set is $\symplex_N$ and the regularization function is chosen to be $\|\cdot\|^2_H$ for some input matrix $H\succ 0$:

\begin{algorithm}[h!]
 	\caption{OMD: Online Mirror Descent}
 	\begin{algorithmic}[1]
 		\STATE \textbf{input:} $H\succ 0$, $\{\eta_t\}_{t=1}^T$, $x_1\in \symplex_N$.
 		\FOR {$t=1$ to $T$}
 			\STATE Play $x_t$
			\STATE Suffer cost $x_{t}\cdot \ell_t$ and observe $\ell_t$
			\STATE Update
 				\[x_{t+1} = \arg\min_{x\in \symplex_N}\ell_t\cdot x + \eta_t^{-1}\|x-x_t\|^2_H.\]
		\ENDFOR
 		%\RET\URN ${\x}_T $
 	\end{algorithmic}
 	\label{alg:omd}
 \end{algorithm}

The regret bound of the algorithm is dependent on the choice of regularization and is given as follows:
%(see, for example, \citealp{hazan2015online}):
%
%\tk{change to general OMD?}

\begin{lemma}[e.g., \citealp{hazan2015online}]
\label{lem:omd}
The $T$-round regret of the OMD algorithm (\algref{alg:omd}) is bounded as follows:
\[
\sum_{t=1}^T \ell_t \cdot x_t - \sum_{t=1}^T \ell_t \cdot x^*
\leq
	\frac{1}{\eta_T} \norm{x_1-x^*}_H^2
	+ \frac{1}{2} \sum_{t=1}^T \eta_t (\norm{\ell_t}_H^*)^2
	~.
\]
\end{lemma}

\subsection{Rademacher Complexity}
\label{sec:radamacher}

Our tool to analyze the stochastic case will be the Rademacher Complexity,
specifically we will use it to bound the regret of a ``Follow The Leader" algorithm (FTL).
Recall that the FTL algorithm selection rule is defined as follows:
\[ x_t = \arg\min_{x \in \symplex_N} \sum_{i=1}^{t-1} \ell_i \cdot x.\]
One way to bound the regret of the FTL algorithm in the stochastic case is by bounding the Rademacher complexity of the feasible samples. Recall that the Rademacher Complexity of a class of target function $\F$ over a sample $S_t=\{\ell_1,\ldots, \ell_t\}$ is defined as follows
\[R(\F,S_t) = \E_{\sigma} \left[ \sup_{f\in \F}   \frac{1}{t}\sum_{i=1}^t \sigma_{i} f(\ell_i)\right],\]
where $\sigma \in \{-1,1\}^t$ are i.i.d.~Rademacher distributed random variables.
The following bound is standard and well known,  and for completeness we provide a proof in \apref{sec:standard}.%
\footnote{Surprisingly, we could not find any specific reference that precisely derives it.}

\begin{lemma}
\label{lem:standard}
Let $K$ be a symmetric convex set centered around zero in $\reals^d$. Recall that the dual set $K^*$ is defined as follows:
\[K^* = \{x:  \sup_{y\in K} |y\cdot x| \le 1\}.\]
Let $S_t=\{\ell_1,\ldots,\ell_t\}\subseteq K$ and let $\F\subseteq \alpha K^*$ be a subclass of linear functions, then:
\[ R(\F,S_t) \le \alpha\sqrt{\frac{d}{t}}.\]
\end{lemma}
Another standard bound applies to the case where $\mathcal{F}$ is bounded in the $l_1$-norm.
\begin{lemma}[\citealp{kakade2009complexity}]
\label{lem:radl1}
Let $S_t=\{\hat{\ell}_1,\ldots,\hat{\ell}_t\}\in \reals^N$ and let $\F_1$
be a subclass of linear functions such that $\sup\{ \|f\|_1 : f\in \F\} \leq 1$, then:
\[ R(\F_1,S_t) \le \max_i \|\hat{\ell_i}\|_\infty \sqrt{\frac{2\log N}{t}}.\]
\end{lemma}

The Rademacher complexity is a powerful tool in statistical learning theory and it allows us to bound the generalization error of an FTL algorithm. Namely, for every sample $S_t=\{\ell_1,\ldots,\ell_t\}$ denote:
\[f_S = \arg\min_{f\in \F} \sum_{i=1}^t f(\ell_i).\]
Then we have the following bound for every $f^* \in \mathcal{F}$ (for i.i.d.~loss vectors; see for example \citealp{shalev2014understanding}):
\[
 \mathop\E_{S_t\sim D}\mathop\E_{\ell\sim D} [f_{S_t}(\ell)- f^*(\ell)]
 \le 2 \mathop\E_{S_t\sim D} [ R(\F,S_t) ] ~.
\]
Applying this to FTL in the experts setting we have, in terms of regret, that for any $x^*$:
\begin{align}
\label{eq:radbound}
\E\left[\sum_{t=1}^T \ell_t \cdot x_t - \ell_t\cdot x^*\right]
&=\sum_{t=1}^T \mathop \E_{\ell_1,\ldots,\ell_{t-1}\sim D} \mathop\E_{\ell_t\sim D}[\ell_t \cdot x_t - \ell_t\cdot x^*]\nonumber\\
&\le  2\sum_{t=1}^T \mathop\E_{S_{t-1}\sim D} [R(\symplex_N,S_{t-1})]~.
\end{align}

\section{Upper Bound}
\label{sec:upper}

%\subsection{Adversarial Online Experts}
%\label{sec:nonstochastic}

In this section we discuss our online algorithm for the adversarial model, which is given in \algref{alg:main}.
The algorithm is a version of Online Mirror Descent with adaptive regularization.
It maintains a positive-definite matrix $H$, which is being updated whenever the newly observed loss vector $\ell_t$ is not in the span of previously appeared losses.
In all other time steps---i.e., when $\ell_t$ remains in the previous span---the algorithm preforms an Online Mirror Descent type update (see \algref{alg:omd}), with the function $\norm{x}_H^2 = x\tr H x$ as a regularizer.

The algorithm updates the regularization matrix $H$ so as to adapt to the low-dimensional geometry of the set of feasible loss vectors.
Indeed, as our analysis below reveals, $H$ is an ellipsoidal approximation of a certain low-dimensional convex set in $\reals^N$ to which the loss vectors $\ell_t$ can be localized.
This low-dimensional set is the intersection of the unit cube in $N$ dimensions---in which the loss vectors $\ell_t$ reside by definition---and the low dimensional subspace spanned by previously observed loss vectors, given by $\mathrm{span}(U)$.
Whenever the latter subspace changes, namely, once a newly observed loss vector leaves the span of previous vectors, the ellipsoidal approximation is recomputed and the matrix $H$ is updated accordingly.

\begin{algorithm}[h!]
\caption{Online Low Rank Experts}
\begin{algorithmic}[1]
\STATE Initialize:  $x_1= \frac{1}{N} \bm{1}_N$ , $\tau = 0$, $k = 0$, $\U = \set{}$
\FOR {$t=1$ to $T$}
	\STATE Observe $\ell_t$, suffer cost $x_{t}\cdot \ell_t$.
	\IF{ $\ell_t \notin \mathrm{span}(\U)$}
		\STATE Add $\ell_t$ as a new column of $\U$, reset $\tau=0$, and set $k \leftarrow k+1$.
		\STATE Compute $M=\mvee{\U^\top}$ and $H=I_n +\U^\top M \U$.
	\ENDIF
	\STATE let $\tau \leftarrow \tau+1$ and $\eta_t = 4\sqrt{k/\tau}$, and set:
	\[x_{t+1} = \arg\min_{x\in \symplex_N}\ell_t\cdot x + \eta_t^{-1}\|x-x_t\|^2_H.\]
\ENDFOR
%\RET\URN ${\x}_T $
\end{algorithmic}
\label{alg:main}
\end{algorithm}

%\subsubsection{Analysis}

To derive \thmref{thm:main}, we begin with analyzing a simpler case where the learner is aware of the subspace from which losses are derived. Specifically, assume that at the beginning of the rounds, the learner is equipped with a rank $d$ matrix $\U$ such that for all losses $\ell_1,\ell_2,\ldots \in \mathrm{span}(\U)$ where we denote by $\mathrm{span}(\U)$ the span of the columns of the matrix $\U$.

In this simplified setting, we can obtain a regret bound of $O(\sqrt{dT})$ via John's theorem (\thmref{thm:john}).%
\footnote{We remark that for the simplified setting, the $O(\sqrt{dT})$ regret bound is in fact tight, as our $\Omega(\sqrt{dT})$ lower bound (given in \secref{sec:lower}) applies in a setting where the subspace of the loss vectors is known a-priori to the learner.}
As discussed above, the loss vectors $\ell_1,\ldots,\ell_T$ can be localized to the intersection of the unit cube in $N$ dimensions with the $d$-dimensional subspace spanned by the columns of $\U$.
Then, John's theorem asserts that the minimal-volume enclosing ellipsoid of the intersection is a $\sqrt{d}$-approximation to the set of feasible loss vectors.

\begin{theorem}
\label{thm:known}
Run \algref{alg:omd} with Input $H$, $\{\eta_t\}$ and $x_1$ defined as follows: (i) $H=I_n + \U^\top M \U$, where $M=\mvee{\U\tr}$,
(ii) $\eta_t= 4\sqrt{d/t}$, where $d=\mathrm{rank}(\U)$, and (iii) $x_1 \in \symplex$ is arbitrary.
If $\ell_1,\ldots, \ell_T \in \mathrm{span}(\U)$ , then the expected $T$-round regret of the algorithm is at most $8\sqrt{dT}$.
\end{theorem}

\begin{proof}
Consider the $d$-dimensional polytope \[P = \set{v \in \reals^d \,:\, \norm{\U^\top v}_\infty \le 1}.\]
Then by John's Theorem (Theorem~\ref{thm:john}),
%$\cE(M)$ is a $2$-approximate MVEE of $P$, and by John's theorem (for symmetric convex bodies)
we have,
\begin{equation}\label{eq:MPM}
	\cE(\tfrac{1}{2d} M) ~\subseteq~ P ~\subseteq~ \cE(M) ~.
\end{equation}
In order to apply \lemref{lem:omd}, we need to bound both $\norm{\ell_t}_H^*$ and $\norm{x_1-x^*}_H^2$.
We first bound the norms $\norm{\ell_t}_H^*$.
Notice that for each loss vector $\ell_t$ there exists $v_t \in P$ such that $\ell_t = \U^\top v_t$ (as $\ell_t \in \mathrm{span}(U)$ and $\norm{\ell_t}_\infty \le 1$).
Thus, we can write,
\[	(\norm{\ell_t}_H^*)^2
=
	\ell_t^\top H^{-1} \ell_t
=
	v_t^\top \U (I_n + \U^\top M \U)^{-1} \U^\top v_t
\leq
	v_t^\top \U (\U^\top M \U)^\dagger \U^\top v_t
=
	v_t^\top M^{-1} v_t
~,
\]
where we have used \lemref{lem:LML} (see \apref{ap:technical}).
Now, since $v_t \in P$ and $\cE(M)$ is enclosing~$P$, we obtain $v_t^\top M^{-1} v_t \le 1$.
This proves that $$(\norm{\ell_t}_H^*)^2 \le 1.$$

Next we bound $\|x_1-x^*\|_H \leq 2$
Since $\|x_1-x^*\|_H \leq 2 \max_{x \in \Del_n} \norm{x}_H$, it suffices to bound
$ \max_{x \in \Del_n} \norm{x}_H$.
Hence, our goal is to show that $\norm{x}_H \le 2\sqrt{d}$ for all $x \in \Del_n$.
Since $\norm{x}_H^2 = 1 + 2d \, \norm{x}_{H'}^2$ with $H' = \tfrac{1}{2d} \U^\top M \U$, it is enough to bound the norm $\norm{x}_{H'}^2$.
Given a convex set in $\reals^d$, recall that the dual set is given by
\[ P^*= \{x: \sup_{p\in P} |x\cdot p| \le 1\}.\]
The dual of an ellipsoid $\cE(M)$ is given by $(\cE( M))^*= \cE(M^{-1})$ and it is standard to show that \eqref{eq:MPM} implies in the dual:
\[ (\cE( M))^* \subseteq P^* \subseteq (\cE( \tfrac{1}{2d}M))^*.\]
Taken together we obtain that $P^*\subseteq  \cE(2d M^{-1})$.
Note that by definition the columns of $U$ are in $P^*$, hence, for every $u_i$,
\[\|u_i\|_M^2\le 2d.\]
Since $x\in \symplex_N$,
\[
\|x\|_{H'}^2= \tfrac{1}{2d}\|Ux\|_M^2\le \tfrac{1}{2d}\max_i \|u_i\|_M^2 \le 1 
~.
\]

Equipped with the bounds $\norm{x}_H \le \sqrt{1+2d}\leq 2\sqrt{d}$ for all $x \in \Del_n$ and $\norm{\ell_t}_H^* \le 1$ for all $t$, we are now ready to analyze the regret of the algorithm, which via \lemref{lem:omd} can be bounded as follows:
\begin{align*}
\regret_T &=	\sum_{t=1}^T \ell_t \cdot x_t - \sum_{t=1}^T \ell_t \cdot x^*\\
&\leq 	\frac{1}{\eta_T} \norm{x_1-x^*}_H^2
	+ \frac{1}{2} \sum_{t=1}^T \eta_t (\norm{\ell_t}_H^*)^2 
& \hfill \because~~ \textrm{\lemref{lem:omd}}\\
&\leq
	\frac{4}{\eta_T} \max_{x \in \Del_n} \norm{x}_H^2
	+ \frac{1}{2} \sum_{t=1}^T \eta_t (\norm{\ell_t}_H^*)^2
& \because~~ \|x_1-x^*\|_H \leq 4 \max_{x \in \Del_n} \norm{x}_H \\
&\leq
	\frac{16d}{\eta_T} + \frac{1}{2} \sum_{t=1}^T \eta_t 
& \because~~ \max_{x \in \Del_n} \norm{x}_H^2 \leq 4d, ~ \norm{\ell_t}_H^*\leq 1 
.
\end{align*}
%the second inequality uses $\|x_1-x^*\|_H \leq 2 \max_{x \in \Del_n} \norm{x}_H$.
A choice of $\eta_t = 4\sqrt{d/t}$, together with the inequality $\sum_{t=1}^T 1/\sqrt{t} \le 2\sqrt{T}$, gives the theorem.
\end{proof}

The $d$-low rank setting does not assume that the learner has access to the subspace $U$, and potentially an adversary may adapt her choice of subspace to the learner's strategy. However, the learner can still obtain regret bounds that are independent of the number of experts. We are now ready to prove \thmref{thm:main}.\\

\begin{proof}[Proof of \thmref{thm:main}]
Let $t_0 = 1$, $t_{d+1} = T$ and for all $1 \le k \le d$ let $t_k$ be the round where the $k$'th column is added to $\U$. Also, let $T_k = t_{k+1}-t_k$ the length of the $k$'th epoch.
Notice that between rounds $t_k$ and $t_{k+1}$ the algorithm's execution is identical to \algref{alg:omd} with input depicted in \thmref{thm:known}. Therefore its regret in this time period is at most $8\sqrt{kT_k}$.
The total regret is then bounded by
\begin{align*}
	8 \sum_{k=0}^d \sqrt{k T_k}
\leq
	8 \sqrt{\sum_{k=0}^d k} \cdot \sqrt{\sum_{k=0}^d T_k}
\leq
	8d \sqrt{T}
~,
\end{align*}
and the theorem follows.
\end{proof}

\subsection{Stochastic Online Experts}
\label{sec:stochastic}

We now turn to analyze the regret in the stochastic model, where the loss vectors $\ell_t$ are chosen i.i.d.~from some unknown distribution.
In this case we can achieve a right regret bound of $O(\sqrt{d T})$ using a simple ``Follow The Leader" (FTL) algorithm.
We will in fact show an even stronger result for the stochastic case, that an approximate rank is enough to bound the complexity.
Recall that the approximate rank, $\epsrank(L)$, of a matrix is defined as follows (see \citealp{alon2013approximate}):
\[{\epsrank}(L)= \min \{ \mathrm{rank}(L') : \|L'-L\|_{\infty} < \epsilon\}.\]
The following statement is the main result for this section:

\begin{theorem}
\label{thm:stochastic}
Assume that an adversary chooses her losses $\{\ell_t\}$ i.i.d.~from some distribution $\mathcal{D}$ supported on $[-1,1]^N$. Then the $T$-round regret of the FTL algorithm is bounded by:
\[ \regret_T \le  8\E\left[ \sqrt{T \cdot \epsrank(L)}\right] + \epsilon \sqrt{T \log N},\]
for every $0\le \epsilon<1$. In particular, if $\mathrm{rank}(L) \le d$ almost surely, then
$\regret_T = O(\sqrt{dT}).$
\end{theorem}
\begin{proof}
Our proof relies on \eqref{eq:radbound} and a bound for $R(\symplex_N,S_t)$.
Fix a sequence $S_T=\{\ell_1,\ldots,\ell_T\}$ and let $d=\epsrank(L)$ and let $\U$ be $N\times d$ matrix such that
\[\L= \U V + \hat{L}~,\]
where $\max_{i,j} |\hat{L}_{i,j}|<\epsilon$. We will denotes the columns of $\hat{L}$ by $\hat{\ell}_1,\ldots \hat{\ell}_N$.
We define a symmetric convex set centered around zero in $\reals^d$ as follows:
\[K= \{v: \sup\nolimits_{i} |u_i \cdot v| \le 2\}~.\]

Note that for every $v_t$ we have that $v_t \in K$ if $\epsilon\leq 1$. By definition of the set we have:
$u_i\in 2 K^*$ for every $i$. One can verify that  $K^*$ is convex, hence if we let $\F=\mathrm{conv}(u_1,\ldots, u_N)$  we have that $\F\subseteq 2 K^*$. We can think of $\F$ as a linear function space, where $f_u(v)= u\cdot v$.
It follows by \lemref{lem:standard} that $R(\F,S_t) \le \sqrt{2d/t}$. Finally,
\[
R(\Delta_N ,S_t) = \E \left[\sup_{x\in \symplex_N} \sum_{i=1}^t\frac{1}{t} \sigma_i x\cdot \ell_i \right]
\le  \E \left[ \sup_{x\in \symplex_N} \frac{1}{t}\sum_{i=1}^t \sigma_i x\cdot \U v_i \right]
+ \E \left[ \sup_{x\in \symplex_N}\frac{1}{t}\sum_{i=1}^t \sigma_i x\cdot \hat{\ell} \right] ~.
\]
Next, we have:
\begin{align}\label{eq:rankbound}
\E \left[ \sup_{x\in \symplex_N}  \frac{1}{t}\sum_{i=1}^t \sigma_i x\cdot \U v_i \right]
&=\E \left[ \sup_{x\in \symplex_N} \frac{1}{t}\sum_{i=1}^t \sigma_i (\U^\top x)\cdot  v_i \right] \nonumber\\
&= \sup_{f\in \mathrm{conv}(u_i)} \E \left[ \frac{1}{t}\sum_{i=1}^t \sigma_i f(v_i) \right] \nonumber\\
&=R(\F,S_t)<2\sqrt{\frac{d}{t}}~.
\end{align}
and by \lemref{lem:radl1},
\begin{equation}\label{eq:noisebound}
\E \sup_{x\in \symplex_N} \left[ \frac{1}{t} \sum_{i=1}^t \sigma_i x\cdot \hat{\ell_i} \right]
\le \epsilon \sqrt{\frac{2 \log N}{t}} ~.
\end{equation}
Taking \eqref{eq:rankbound} and \eqref{eq:noisebound} together, we have:
\[R(\Delta_N ,S_t)  \le 2\sqrt{\frac{\epsrank{L}}{t}}+\epsilon \sqrt{\frac{2 \log N}{t}}~.\]
The statement now follows from \eqref{eq:radbound}.
\end{proof}

\section{Lower Bound}
\label{sec:lower}

We now prove \thmref{thm:lowbound}. For our proof we will rely on lower bounds for online learning of hypotheses classes with respect to the Littlestone dimension (see \citealp{shalev2011online}). For a class $\mathcal{H}$ of target functions $h:\mathcal{X}\to \{0,1\}$, the Littlestone dimension $\ldim(\mathcal{H})$ measures the complexity, or online learnability, of the class.

To define $\ldim(\mathcal{H})$ one considers trees whose internal nodes are labeled by instances. Any branch in such a tree can be described as a sequence of examples $(x_1,y_1),\ldots, (x_d,y_d)$ where $x_i$ is the instance associated with the $i$th node in the path, and $y_i$ is $1$ if  $x_{i+1}$ is the right child of the $i$--th node, and $y_i=0$ otherwise.
$\ldim(\mathcal{H})$ is then defined as the depth of the largest binary tree that is shattered by $\mathcal{H}$. An instance-labeled tree is said to be shattered by a class $\mathcal{H}$ if for any root-to-leaf path $(x_1,y_1),\ldots,(x_d,y_d)$ there is some $h\in \mathcal{H}$ such that $h(x_i)=y_i$.

To prove \thmref{thm:lowbound}, we need the following result about the Littlestone dimension.

\begin{lemma}[\citealp{ben2009agnostic}]
\label{lem:pal}
Let $\mathcal{H}$ be any hypothesis class with finite $\ldim(\mathcal{H})$, where $\ldim$ is the Littlestone-dimension of a class $\mathcal{H}$.
For any (possibly randomized) algorithm, there exists a sequence of labeled instances $(v_1,y_1),\ldots, (v_T,y_T)$ with $y_t\in \{0,1\}$ such that
\[\E\left[ \sum_{t=1}^T |\hat{y}_t - y_t| \right]  - \min_{h\in \mathcal{H}} \sum_{i=1}^T |h(x_t)-y_t| \ge \sqrt{\frac{\ldim(\mathcal{H})T}{8}}~.\]
where $\hat{y}_t$ is the algorithm's output at iteration $t$.
\end{lemma}

\begin{proof}[Proof of \thmref{thm:lowbound}]
We let $\mathcal{H}$ be the $2^d$ vertices of the $d$-dimensional hypercube.
We define a function class $\mathcal{F}$ over the domain $\mathcal{X} = \{e_1,\ldots, e_d\}$ of standard basis vectors.
A function $f_u \in \mathcal{F}$, is labeled by $u \in \mathcal{H}$, and defined over the set of basis vector $e_j$, as follows,
\[
f_u(e_j) = \begin{cases} 
	0 & \text{if~~$u(j)=-1$,} \\ 
	1 & \text{if~~$u(j)=1$.}
	\end{cases}
\]
One can verify that $\ldim(\mathcal{F})=d$.
For each $u_i \in \mathcal{H}$ and $y\in \{0,1\}$, we can write $$|f_{u_i}(e_j) -y| = \frac{1-(2y-1)\cdot u_i\cdot e_j}{2}.$$
By \lemref{lem:pal},
%and by proper change of variables from $\{0,1\}$ domain to $\{-1,1\}$
we deduce that for any algorithm,
% (as in \thmref{thm:lowbound})
there exists a sequence $(v_1,\bar{y}_1),\ldots,(v_T,\bar{y}_T)$ of standard basis vectors $v_1,\ldots,v_T$ and $\bar{y}_1,\ldots \bar{y}_T \in \{-1,1\}$ such that:
\begin{equation}\label{eq:lowbound}
\sum_{t=1}^T \sum_i x_t(i) u_i \cdot (\bar{y}_t v_t ) - \min_{u} \sum_{i=1}^T u\cdot (\bar{y}_t v_t )  \ge 2\sqrt{\frac{dT}{8}}~.
\end{equation}
We now consider an adversary that chooses $U$ as his expert matrix, and at round $t$ the learner observes $\ell_t = U(\bar{y}_t v_t)$. The lower bound now follows from \eqref{eq:lowbound}; the fact that $\mathrm{rank}(\L)=d$ follows from the fact that our experts are embedded in $\reals^d$.
\end{proof}

\section{Discussion and Open Problems}

% A working premise in Online Learning is that in most cases the stochastic case is the hardest case.
% Indeed, the literature is filled with generalization bounds and their analogue regret bounds.
% However, a striking difference is that the statistical bounds are often achieved using simple ERM algorithms,
% that are oblivious to any structure in the problem, even if that allows the generalization bounds to be valid.
% In contrast, to achieve the analogue regret bound, one has to work harder.
% For finite hypothesis class the $\log N$ factor is achieved by a sophisticated tailored algorithm,
% and for more general convex problems in Euclidean space a tailored regularization needs to be invoked in order to achieve optimality.
% Thus, a key difference is that online algorithms need to be tailored to the structure of the problem. This lead to the disappointing fact that to achieve optimal regret bounds, it is not enough for the problem to be structured but the learner needs to actively understand the structure.
%
% Our problem is an attempt to better understand this key difference: Namely, we wish to understand if an online linear predictor can somehow exploit the geometry of the problem in an implicit manner, similar to an FTL algorithm.
% For this, we invoke a setting where the learner must choose its predictor without the a-priori ability to devise a regularizer.
% Our finding so far indeed demonstrate that even without access to the structure the learner can indeed overcome her dependence on the irrelevant parameter $N$.

We considered the problem of experts with a hidden low rank structure. Our findings are that in the non-stochastic case, similar to the stochastic case, the regret bounds are independent of the number of experts. 
The most natural question is then to bridge the gap between the upper and lower bounds:
\begin{question}
Is there an algorithm that can achieve regret $O(\sqrt{dT})$ for any sequence $\ell_1,\ldots,\ell_T$ such that $\mathrm{rank}(L)=d$?
Alternatively, can one prove a lower bound of $\Omega(d\sqrt{T})$?
\end{question}

As discussed, our agenda is more general than the low-rank setting. Our aim is to construct new online algorithms that can exploit structure in the data, without explicit information on the structure. Other settings can also be considered within our framework.

Another interesting setting, that avoids dependence in dimension, is to assume that experts are embedded in a Hilbert space. By isomorphisms of Hilbert spaces this is equivalent to an adversary that chooses an expert embedding matrix $U \in \reals^{N\times N}$ such that for every $u_i$ we have $\|u_i\|_2\leq 1$ and correspondingly at each time point we receive a vector $v_t$ such that $\|v_t\|_2\leq 1$ as a result we have a factorization:
\[L = UV^\top, \qquad \|U\|_{2,\infty},\|V\|_{2,\infty} \leq 1,\]
where $\|X\|_{2,\infty} = \sup_{\|y\|\leq 1} \|X y\|_\infty$.
Recall the definition of the max-norm, also called the $\gamma_2$-norm  \citep{srebro2005rank}:
\[\|L\|_{\max} = \min_{UV^\top=L}  \|U\|_{2,\infty}\cdot \|V\|_{2,\infty} .\]
Thus, similar to the low rank setting we can define this setting as follows: At each round a learner chooses $x_t \in \symplex_N $, an adversary replies by choosing a loss vector $\ell_t$, and the learner incurs the corresponding loss. The adversary is restricted to strategies such that
\[ \|L\|_{\max} \le 1.\]
The importance of this setting is that the proper generalization bound for this case is dimension independent (e.g., \citealp{kakade2009complexity}). Hence, we ask the following question:
\begin{question}
Is there an algorithm that can achieve regret $O(\sqrt{T})$ for any sequence $\ell_1,\ldots,\ell_T$ such that $\|L\|_{\max}\le 1$?
\end{question}
We can also generalize this setting to any pair of norms, $\|\cdot\|$ and its dual $\|\cdot\|^*$, where the description of the game remains the same. The adversary chooses an embedding $U$ of the experts with bounded $\|\cdot \|$ norm. Then, at each round he chooses a set of vectors $v_t$ with $\|\cdot \|^*$ bounded norm.
%The learner observes a sequence $\ell_1,\ldots, \ell_T$ and wishes to compete with the stochastic bound guaranteed for an ERM algorithm in the $\|\cdot\| - \|\cdot\|^*$ setting. \ym{Seems redundant, not much content in this paragraph}

Finally, a different interesting direction to pursue in future work is to extend the noisy result to the adversarial setting. Namely,
\begin{question}
Is there an algorithm that can achieve regret $O(\sqrt{dT}+ \epsilon\sqrt{T\log N} )$ for any sequence $\ell_1,\ldots,\ell_T$ such that $\epsrank(L) \leq d$?
\end{question}

\bibliographystyle{abbrvnat}
\bibliography{loraex}

\appendix

\section{Technical Proofs}
\label{ap:technical}

\begin{lemma} \label{lem:LML}
Let $M \in \reals^{d \times d}$, $U \in \reals^{d \times n}$ such that $M \succ 0$ and $\U$.
Then \[\U (\U\tr M\U)\pinv \U\tr = M^{-1}.\]
\end{lemma}

\begin{proof}
Let $N = M^{1/2} \U$. Then, we have $N (N^\top  N)^\dagger N^\top = I_d$.
To see this, write the SVD decomposition $N = O \Sigma V^\top$ with diagonal non-singular $\Sigma \in \reals^{d \times d}$ and $OO\tr = O\tr O = V\tr V = I_d$.
Then,
\[
	N (N\tr N)\pinv N\tr
=
	O\Sigma V\tr (V \Sigma^2 V\tr)\pinv V \Sigma O\tr
=
	O \Sigma V\tr (V \Sigma^{-2} V\tr) V \Sigma O\tr
=
	I_d.
\]
Expanding the definition of $N$, we get $M^{1/2} U (\U\tr M \U)\pinv \U\tr M^{1/2} = I_d$, and since $M^{1/2}$ is non-singular, we can multiply by $M^{-1/2}$ on both sides and obtain the lemma.
\end{proof}

\subsection{Proof of \lemref{lem:standard}}
\label{sec:standard}

The proof relies on the following corollary of John's Theorem:
\begin{lemma}\label{lem:smalljohn}
Let $K$ be a symmetric convex set centered around zero in $\reals^d$. There exists a positive semi-definite matrix $\Sigma$ such that for every $x\in K$:

\[ x^\top \Sigma x \le \sup_{f\in K^*}| f(x)|^2 \le d (x^\top \Sigma x) ~.\]
\end{lemma}ד

\begin{proof}(of \lemref{lem:standard}).
wlog we assume $\alpha=1$, the general case follows since $R(\alpha \F,S) = \alpha R(\F,S)$.
We have
\[R(\F,S) = \mathbb{E_{\sigma}}\left[\sup_{f\in \F} \frac{1}{t}\sum_{i=1}^t \sigma_{i} f(\ell_i)\right]
=
\mathbb{E_{\sigma}}\left[\sup_{f\in \F}  f( \frac{1}{t}\sum_{i=1}^t \sigma_{i} \ell_i)\right]
\le
\sqrt{\mathbb{E_{\sigma}}\left[\sup_{f\in \F} f^2 ( \frac{1}{t}\sum_{i=1}^t \sigma_{i} (\ell_i))\right]}.
\]
Next, we take $\Sigma$ whose existence follows from \lemref{lem:smalljohn}. Note that $\Sigma$ defines a scalar product. Specifically let us denote
$\langle \ell_i,\ell_j \rangle = \ell_i^\top \Sigma \ell_j,$ and also we let $\|\ell_i\|_2^2 = \ell_i^\top \Sigma \ell_i$. Then we have
\begin{align*}
\sqrt{\mathbb{E_{\sigma}}\left[ \sup_{f\in K^*}f^2 \left(\frac{1}{t}\sum_{i=1}^t \sigma_{i} \ell_i \right)\right]}
&\le
\sqrt{d\,\mathbb{E_{\sigma}}\left[ \frac{1}{t^2} \left\| \sum_{i=1}^t \sigma_{i} \ell_i \right\|_2^2\right]} \\
&=
\sqrt{d\,\mathbb{E_{\sigma}}\left[ \frac{1}{t^2}\sum_{i,j=1}^t \sigma_{i}\sigma_j \langle \ell_i,\ell_j \rangle \right]}\\
&=
\sqrt{d\,\E_{\sigma} \left[ \frac{1}{t^2} \sum_{i=1}^t \sigma^2\|\ell_i\|_2^2 \right]} =
\sqrt{\frac{d}{t^2} \sum_{i=1}^t \|\ell_i\|_2^2} \\
&\le
\sqrt {\frac{d}{t} \max_i \|\ell_i\|_2^2}
\le
\sqrt{\frac{d}{t}\max_i \sup_{f\in K^*} f^2(\ell_i)}
\le
\sqrt{\frac{d}{t}} ~,
\end{align*}
as claimed.
\end{proof}

\section{Lower Bounds for the AdaGrad Algorithm}
\label{sec:adagrad}

AdaGrad (see \algref{alg:adagrad}) is an algorithm that adapts the regularization matrix with respect to prior losses. Our aim in this section is to show that this learning scheme of the regularization cannot lead to a regret bound that is independent of the number of experts. Our strategy is as follows: since the AdaGrad algorithm depends on a learning rate parameter $\eta$ we consider two cases: either $\eta$ scales with $N$ and becomes smaller, but then we show that for some sequence the algorithm's update is ``too slow". On the other hand, we show that if $\eta$ does not scale with $N$, the algorithm becomes less stable, and we can again inflict damage. Taken together we prove the following statement:
\begin{theorem}
\label{thm:adagrad}
   Consider \algref{alg:adagrad}. For concreteness we assume that $x_1=\frac{1}{N} \bf{1}$.  For sufficiently large $N$, if $T<\sqrt{N}/6$ then there exist a sequence $\{\ell_1,\ldots, \ell_T\}$ such that $\regret_T \ge T/2$ and $\mathrm{rank}(L) = 1$.
\end{theorem}
\begin{algorithm}[h!]
\caption{AdaGrad}
\begin{algorithmic}[1]
	\STATE Input: $\eta,\delta, x_1\in\symplex_N$.
	\STATE Initialize: $S_0=G_0=\delta I$
	\FOR {$t=1$ to $T$}
		\STATE Observe $\ell_t$, suffer cost $x_{t}\cdot \ell_t$.
	\STATE set
		\[S_t= S_{t-1} + \ell_t\ell_t^\top, ~ G_t = S_t^{1/2} \]
		\[y_{t+1} = x_t - \eta G_t^{-1} \ell_t\]
		\[x_{t+1} = \arg\min_{x\in \symplex_N} \|y_{t+1} - x\|_{G_t}^2\]
	\ENDFOR
	%\RET\URN ${\x}_T $
\end{algorithmic}
\label{alg:adagrad}
\end{algorithm}

\begin{lemma}
\label{lem:adagrad}
Consider \algref{alg:adagrad} with arbitrary $\eta$ and $\delta$. For concreteness we assume that $x_1=\frac{1}{N} \bf{1}$. 
For sufficiently large $N$, if $T\le \max\big(\frac{1}{36\eta^2} +2\frac{\sqrt{\delta}}{6\eta}, \eta^2 N -\delta \big)$ then there exist a sequence $\ell_1,\ldots, \ell_T$ such that $\regret_T \ge T/2$ and $\mathrm{rank}(L) = 1$.
\end{lemma}

\begin{proof} We prove each bound separately.

\paragraph{Case 1: $T<\frac{1}{36\eta^2} +2\frac{\sqrt{\delta}}{6\eta}$.}

We let $\ell_t = \ee= (-1,\frac{1}{N-1},\frac{1}{N-1},\frac{1}{N-1},\ldots \frac{1}{N-1})$ for all $t$. For every $t$ we have that
\[G_t = \sqrt{ t \ee\ee^\top +\delta I}\] 
and
\[ \eta G_{t}^{-1} \ell_t = \frac{\eta}{\sqrt{t+\delta}\|\ee\|} \ee .\]
Next we use the inequality:
\[ \sum_{t=1}^T \frac{1}{\sqrt{t+\delta}} \le \int_{0}^T \frac{1}{\sqrt{t+\delta}}dt=2\left(\sqrt{T+\delta} - \sqrt{\delta}\right).\]
For $T \le\left(\frac{1}{6\eta}+\sqrt{\delta}\right)^2-\delta=\frac{1}{36\eta^2} +2\frac{\sqrt{\delta}}{6\eta}$, we have that:
\[\left(\sqrt{T+\delta} -\sqrt{\delta}\right) <\frac{1}{6\eta};\]
\[\frac{1}{N} + \frac{\eta}{\|\ee\|} \sum_{i=1}^T \frac{1}{\sqrt{t+\delta}}\le  \frac{1}{N} + \frac{2\eta}{\|\ee\|} \left(\sqrt{T+\delta} -\sqrt{\delta}\right) \le \frac{1}{2},\]
where last inequality follows since $\|\ee\|>1$ and we assume $N \ge 6$.
 One can observe that our update rule does not take $y_t$ out of the simplex $\symplex_N$ and we have
 \[ x_t = \frac{1}{N} \mathbf{1} - \frac{\eta}{\|\ee\|} \sum \frac{1}{\sqrt{t+\delta}} \ee,\]
and further, $x_t(1) < \frac{1}{2}$. In hindsight $x_t(1)$ suffers loss $-T$ while all other experts suffer positive loss. Hence the algorithm's regret is at least

\[\regret_T \ge \frac{T}{2}.\]

\paragraph{Case 2: $T< \eta^2 N -\delta$.}

We now choose $\ee=(\underbrace{+1,+1,+1,+1,+1,+1}_{\text{$N/2$ times}}, \underbrace{-1,-1,-1,-1,-1,-1}_{\text{$N/2$ times}})$. and let
\[\ell_t = (-1)^{t+1} \ee.\]
As before note that
\[\eta G_t^{-1} \ell_t = \frac{(-1)^{t+1}\eta}{\sqrt{t+\delta}\|\ee\|} \ee= \frac{(-1)^{t+1}\eta}{\sqrt{N(t+\delta)}}\ee.\]
We claim that for $\sqrt{t+\delta}< \eta\sqrt{N}$ and $t>1$ we have that:

\begin{equation}\label{eq:zigzag}
x_t= \frac{2}{N} \begin{cases}
(\underbrace{1,1,1,1,1,1}_{\text{$N/2$ times}}, \underbrace{0,0,0,0,0,0}_{\text{$N/2$ times}}) & \text{$t$ is even,} \\
( \underbrace{0,0,0,0,0,0}_{\text{$N/2$ times}},\underbrace{1,1,1,1,1,1}_{\text{$N/2$ times}}) & \text{$t$ is odd.}
\end{cases}
\end{equation}
Hence $x_t \ell_t =1$ and since the cumulative loss of each expert is at most $1$ we have that:
\[\regret_T \ge \frac{T}{2}.\]
To see that \eqref{eq:zigzag} holds, we will show the statement for $x_2$ other cases are easier and follow the same proof:
$y_1 =\mathbf{1}\frac{1}{N} + \alpha \ee$, where $|\alpha|>\frac{1}{\sqrt{N}}$, hence it has the form
\[
y_t= (\underbrace{a,a,a,a,a,a}_{\text{$N/2$ times}}, \underbrace{-b,-b,-b,-b,-b}_{\text{$N/2$ times}})
\]
where $a-b=2/N$ and $a,b>0$. The statement now follows from \lemref{lem:projection} (see below).
\end{proof}

\begin{proof}[Proof of \thmref{thm:adagrad}]
By \lemref{lem:adagrad} we need to show that
\[
\min_{\eta,\delta} \max\left(\frac{1}{9\eta^2} +2\frac{\sqrt{\delta}}{3\eta}, \eta^2 N -\delta\right) > \frac{\sqrt{N}}{6} ~.
\]
To prove this, we note that since both terms in the $\max$ are monotone in both variables, the minimum is attained when there is equality, i.e., the minimal $\eta,\delta$ satisfy:
\[ \frac{1}{36\eta^2} +2\frac{\sqrt{\delta}}{6\eta}=\eta^2 N -\delta.\]
Since $\frac{1}{36\eta^2} +2\frac{\sqrt{\delta}}{6\eta}+\delta = \left(\frac{1}{6\eta}+\sqrt{\delta}\right)^2$, we get:
\[\sqrt{N} = \frac{1}{\eta} \left(\frac{1}{6\eta}+\sqrt{\delta}\right),\] and we have that:
\[
\frac{\sqrt{N}}{6}= \frac{\sqrt{2}}{36\eta^2}+\frac{\sqrt{\delta}}{6\eta} <\frac{\sqrt{2}}{36\eta^2}+2\frac{\sqrt{\delta}}{6\eta} 
~. \qedhere 
\]
\end{proof}
It remains to prove \lemref{lem:projection}, that was used for the proof of \lemref{lem:adagrad}.
\begin{lemma}
\label{lem:projection}
Let $y= (\underbrace{a,a,a,a,a,a}_{\text{$N/2$ times}}, \underbrace{-b,-b,-b,-b,-b}_{\text{$N/2$ times}}) $ where $a,b\ge 0$ and assume that $a-b=2/N$. 
Let $G_t = \sqrt{\delta I + \alpha \ee \ee^\top}^{-1/2}$ for some $\alpha>0$, where
\[
\ee=(\underbrace{+1,+1,+1,+1,+1,+1}_{\text{$N/2$ times}}, \underbrace{-1,-1,-1,-1,-1,-1}_{\text{$N/2$ times}})~.
\] 
Then
\[ 
\min_{x\in\symplex_N} \frac{1}{2} {\|y-x\|^2}_{G_t}=\frac{2}{N} (\underbrace{1,1,1,1,1,1}_{\text{$N/2$ times}}, \underbrace{0,0,0,0,0}_{\text{$N/2$ times}})~.
\]
\end{lemma}
\begin{proof}
Considering the Lagrangian and KKT conditions, we observe that $x$ minimizes the distance iff the following hold:
\begin{enumerate}
\item $x\in \symplex_N$ (primal feasibility)
\item $\lambda\succ 0$ and $\theta(1)=\theta(2)=\cdots=\theta(N)$. (dual feasibility)
\item $x= y+ G_t^{-1}(\lambda + \theta)$ (stationarity)
\item $x(i) \ne 0 \Rightarrow \lambda(i)=0$ and $\lambda(i)\ne 0 \Rightarrow x(i)=0$. (complementary slackness)
\end{enumerate}
Next note that $\ee$ is an eigenvector of $G_t$ and we have for some $c<0$ that \[G_{t}^{-1} c\ee = (\underbrace{-b,-b,-b,-b,-b,-b}_{\text{$N/2$ times}}, \underbrace{+b,+b,+b,+b,+b}_{\text{$N/2$ times}}).\]
Now we can write
\[ c\ee = \underbrace{(\underbrace{0,0,\ldots ,0,0}_{\text{$N/2$ times}}, \underbrace{-2c,-2c,\ldots,-2c,-2c}_{\text{$N/2$ times}})}_{\lambda}+ \underbrace{(\underbrace{c,c,c,c,c,c}_{\text{$N/2$ times}}, \underbrace{c,c,c,c,c,c}_{\text{$N/2$ times}})}_{\theta},
\]
that concludes the proof.
\end{proof}

\end{document}